\documentclass[11pt]{amsart}
\usepackage{fullpage}
\usepackage{dsfont}
\usepackage{amsmath,amssymb}
 \usepackage[foot]{amsaddr}
\usepackage{listings}
\usepackage{xcolor}
\usepackage{graphics,graphicx}
\usepackage{hyperref}
\usepackage{subfigure}

\usepackage{epstopdf}
\usepackage[ruled,linesnumbered,algo2e]{algorithm2e}

\newtheorem{theorem}{Theorem}[section]
\newtheorem{assumption}[theorem]{Assumption}
\newtheorem{lemma}[theorem]{Lemma}
\newtheorem{definition}[theorem]{Definition}
\newtheorem{proposition}[theorem]{Proposition}
\newtheorem{corollary}[theorem]{Corollary}
\newtheorem{remark}[theorem]{Remark}

\newcommand*{\R}{\mathbb{R}}

\newcommand*{\One}{\mathds{1}}
\newcommand*{\Po}{\mathrm{\Pi}_\One}
\newcommand*{\Pop}{\mathrm{\Pi}_\One^\perp}
\newcommand*{\LN}{\mathbf{LN}}
\newcommand*{\E}{\mathbb{E}}
\newcommand{\Sim}{\mathbf{t}_{sim}}
\newcommand{\Div}{\mathbf{t}_{div}}
\newcommand*{\Prob}{\mathbb{P}}
\newcommand{\spanOne}{\mathrm{span}\{\One\}}

\title{
Why ``classic" Transformers are shallow \\
and how to make them go deep
}
\author{Yueyao Yu}
\address{School of Science and Engineering,  The Chinese University of Hong Kong, Shenzhen}
\email{yueyaoyu@link.cuhk.edu.cn}
  
\author{Yin Zhang}
\address{School of Data Science, The Chinese University of Hong Kong, Shenzhen}
\email{yinzhang@cuhk.edu.cn}

\date{\today}  % Activate to display a given date or no date

\begin{document}
\maketitle

%\tableofcontents\newpage

\begin{abstract}
%\textcolor{blue}
{Since its introduction in 2017, Transformer has emerged as the leading neural network architecture, catalyzing revolutionary advancements in many AI disciplines.  The key innovation in Transformer is a Self-Attention (SA) mechanism designed to capture contextual information.  However, extending the original Transformer design to models of greater depth has proven exceedingly challenging, if not impossible.  Even though various modifications have been proposed in order to stack more layers of SA mechanism into deeper models, a full understanding of this depth problem remains lacking.
In this paper, we conduct a comprehensive investigation, both theoretically and empirically, to substantiate the claim that the depth problem is caused by \emph{token similarity escalation}; that is, tokens grow increasingly alike after repeated applications of the SA mechanism.  Our analysis reveals that, driven by the invariant leading eigenspace and large spectral gaps of attention matrices, token similarity provably escalates at a linear rate.
%
%\textcolor{blue}
Based on the gained insight, we propose a new strategy of surgically removing excessive similarity in contrast to the existing approach of diminishing the SA mechanism explicitly or implicitly (such as in pre-norm transformers). Preliminary experimental results confirm the effectiveness of the proposed strategy in small-scale post-norm Transformer models.
}

\end{abstract}

%%%%%%%%%%%%%%%%%%%
%%%%%%%%%%%%%%%%%%%
\section{Introduction}

The Transformer architecture~\cite{vaswani2017attention} for neural networks, incorporating self-attention (SA) mechanisms~\cite{vaswani2017attention},  residual connections~\cite{He_2016_CVPR}, layer normalizations~\cite{neyshabur2015norm} and conventional feedforward networks, has revolutionized various areas of AI, including natural language processing, computer vision and beyond~\cite{bommasani2021opportunities, brown2020language, dosovitskiy2020image,  kenton2019bert, liu2021swin, vaswani2017attention}.  One of the key strengths of Transformers lies in their scalability, enabling significant performance improvements through the use of larger models, more data, and increased computational resources~\cite{brown2020language, radford2019language, touvron2023llama}.  
\textcolor{black}{However, further increasing the depth of transformers is by no means a task without obstacles.  Some authors, including  \cite{ethayarajh2019contextual, gao2018representation},  have observed that the representation power of Transformer-based deep models is rather limited to the extent that the learned embeddings only occupy a small portion of the representation space.} %a phenomenon often referred to as Token Similarity.}

\subsection{Token Similarity}

Recent investigations on deep Transformer models, including but not limited to \cite{dong2021attention,  ethayarajh2019contextual, gao2018representation, li2020sentence, mu2018all, noci2022signal,   yan2022addressing}, have shed light on the occurrence of a phenomenon that has been called by different names, including \textit{token uniformity}, \textit{rank collapse}, \textit{representation degeneration} and \textit{representation anisotropy}.  
In this paper, we will use the term \textit{token similarity} with a precise and distinctive definition.
In plain language, token similarity means that as a representation matrix $X\in \R^{n \times d}$ traverses through layers of a Transformer model, the $n$ rows, or tokens as they are called, in $X$ grow increasingly similar to each other, thus substantially reducing the model's expressive capacity and hindering the training of the model~\cite{noci2022signal}. 

To quantitatively measure token similarity, researchers have proposed a number of measurement methods. For example, the following cosine similarity~\cite{ethayarajh2019contextual} calculates the average cosine similarity between all pairs of tokens,
\begin{equation}
    \mathbf{t}_{cos}(X) := \frac{2}{n^2-n} \sum_{i=1}^n \sum_{j>i}^n \frac{x_i^T x_j }{\|x_i\| \|x_j\|},
\end{equation}
where $x_i \in \R^d$ is the $i$-th row of $X\in \R^{n \times d}$ and $\|\cdot\|$ is the Euclidean norm.

To facilitate our analysis, we introduce the following definition of token similarity for any given representation matrix $X\in \R^{n \times d}$.
\begin{definition}
Given any non-zero matrix $X \in \R^{n \times d}$, the token similarity of $X$ is
\begin{equation*} \label{eq: unif}
%\mathrm{Similarity}(X) \equiv  
\Sim(X) := {\|\Po X\|_F^2}/{\|X\|_F^2} ~\in~ [0,1] 
\end{equation*}
where $\Po = \One\One^T/n\in\R^{n\times n}$ and $\One$ is the vector of all ones in $\R^n$.
The token diversity of $X$ is
\[
%\mathrm{Diversty}(X) \equiv  
\Div(X) := {\|\Pop X\|_F^2}/{\|X\|_F^2} ~\in~ [0,1],
\]
where $\Pop = I-\Po\in\R^{n\times n}$. Clearly, the similarity and diversity of $X$ sum up to unity. 
For convenience, we will occasionally also use the notation: $\Pi_1=\Po$ and $\Pi_2=\Pop$.
\end{definition}

We observe that $\Sim(X)=1$, or $\Div(X)=0$, if and only if $X = \One v^T$ for some $v\in\R^d$; that is, all rows of $X$ are the same. 

We start with examining how token similarity evolves in two well-known Transformer models: BERT~\cite{kenton2019bert} and ALBERT~\cite{lan2019albert}, both following the original (or classic) encoder architecture as proposed by \cite{vaswani2017attention}.   
%\textcolor{blue}
{We utilized the package Hugging~Face~\cite{wolf2019huggingface} to conduct this experiment where we set the model depth to 100 for both BERT and ALBERT and initialize model weights by the package default.}
In Figure~\ref{fig:similarity}, we plot token similarity $\Sim(X)$, cosine similarity $\mathbf{t}_{cos}(X)$ and, for BERT model, gradient norm (with respect to model weights) at each Transformer layer (or block, as is often called).\footnote{We excluded ALBERT from the gradient norm experiment since it uses the same set of weights throughout all blocks.}
From Figure~\ref{fig:similarity}, we observe that, as the depth goes deep,  in either model both token similarity and cosine similarity escalate to unity, where the escalation patterns for the two measures appear almost identical.  We also observe that the gradient norm value in BERT remains roughly at a constant level, which indicates that in this case gradient vanishing or exploding is not occurring. 
%\textcolor{blue}

\begin{figure}[ht]
% \begin{wrapfigure}{r}{0.6\textwidth}
	\begin{center}
		\vspace{-.1cm}		
        \includegraphics[width=.9\textwidth,trim=0 0 0 0, clip]{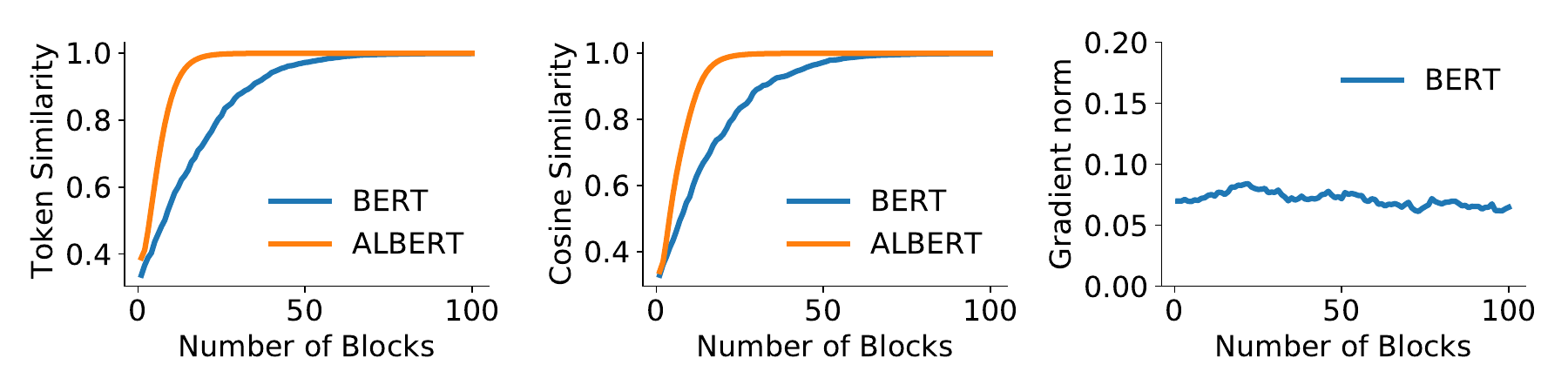}
		\vspace{-.1cm}
		\caption{Values of token similarity~(left), cosine similarity~(middle), 
		and gradient norm~(right) at each block in 2 Transformer models at default initialization. }
		\label{fig:similarity}
	\end{center}
	\vspace{-.2cm}
% \end{wrapfigure}
\end{figure}

%\textcolor{blue}
{As mentioned earlier, it remains to be fully understood why classic Transformer architecture yields outstanding results with shallow models but becomes ``useless" with deeper ones~\cite{takase2023b2t}.  Our experiments in Figure~\ref{fig:similarity} strongly suggest that the root cause be token similarity escalation. In this paper, we conduct a systematic analysis to substantiate this assertion and to address two fundamental questions: why token similarity escalates and how fast it escalates.  
Our theoretical and empirical results will provide definitive answers to these questions.}

%\textcolor{blue}
{In the sequel, we will use capitalized Transformer(s) to refer to models based on the ``classic" architecture as originally proposed by~\cite{vaswani2017attention} (focusing on the encoder side), while subsequently proposed modifications will be referred to as transformers in small letters.}

\subsection{Related Works}
%\textcolor{blue}
A number of recent works have noticed depth-related problems with deep Transformers; for example, the learned word embeddings may degenerate into lying in a narrow cone~\cite{ethayarajh2019contextual}.  \cite{dong2021attention} prove that in pure self-attention models without skip connections, the representation matrix would converge to a rank-one matrix with identical rows, a situation that they call token uniformity, while they also claim that skip-connections should be able to resolve this problem.  However, more recent works~\cite{yan2022addressing, noci2022signal} for example, have observed that residual models still encounter the same token uniformity problem.  In this paper, we will use the term token similarity in order to utilize a distinct definition and avoid possible confusion with previously defined measures.

%\textcolor{blue}
A variety of approaches have been explored to tackle difficulties caused by or related to token similarity. Some works have incorporated regularization terms into the training objectives~\cite{gao2018representation, wang2019improving, zhang2020revisiting} or introduced contrastive learning~\cite{gao2021simcse, qiu2022contrastive} to alleviate the so-called anisotropy problem.  Additionally, researchers have adopted post-processing strategies to normalize word or sentence embeddings~\cite{ huang2021whiteningbert, mu2018all} or transform learned representations into alternative distributions~\cite{li2020sentence, yan2022addressing} to obtain more isotropic representations. Moreover, the authors in \cite{noci2022signal} use so-called residual branch scaling~\cite{bachlechner2021rezero} to slow down rank collapse. The authors in \cite{he2023deep} propose an approach termed ``next-token prediction'' to slow down rank collapse in skip-less Transformers (at a cost of prolonged training time).  In essence, these techniques reduce, explicitly or implicitly, the role of self-attention relative to other operations.

Some researchers try to explain difficulties in deeper Transformers from the perspective of gradient instability. For example, in \cite{wang2019learning, liu2020understanding, xiong2020layer}, the authors conclude that training deep Transformers of the original design is unstable, resulting in bad performance.
%\textcolor{red}
{Recently, the authors in \cite{noci2022signal} prove that when $X$ is rank-one, then certain gradient components would vanish at initialization. the authors in \cite{zhai2023stabilizing} observe that the gradient is unstable when the so-called attention entropy collapses. On the other hand,  the authors in \cite{takase2023b2t} give empirical evidence that at least in an encoder, the gradient norms in different blocks do not rise or fall significantly.}

In the literature, the original Transformer architecture is now often called the post-norm (or post-LN) architecture where layer normalizations are applied after residual connections are added \cite{vaswani2017attention}.   To address difficulties in deep post-norm models, a variant called pre-norm (or pre-LN) architecture has been proposed~\cite{wang2019learning, xiong2020layer}, where layer normalizations are applied to the input of each sub-layer. It appears that contemporary large language models, for example in \cite{brown2020language, touvron2023llama}, have widely adopted the pre-norm architecture.  Nevertheless, with shallow models a noticeable reduction in generalization performance has been reported for pre-norm models in comparison to their post-norm counterparts~\cite{wang2019learning,xiong2020layer}.   In this paper, our analysis and experiments concentrate primarily on the classic, post-norm architecture.  

%%%%%%%%%%%%%%
\subsection{Contributions}

Despite some theoretical works~\cite{dong2021attention, noci2022signal}  on factors influencing token similarity, our literature review suggests that there has been no analysis that accurately quantifies the token similarity dynamics in the original Transformer architecture: why escalation starts and how fast it develops. This work provides such a quantitative analysis for the first time.

\begin{itemize}

    \item     
    %\textcolor{blue}
    {By analyzing a well-defined measure of token similarity dynamics, we prove that the SA-plus-residual sub-layer in the Transformer architecture increases token similarity by default upon standard initialization and in expectation. Our theory, together with strong empirical evidence, reveals why token similarity escalates and how fast it does so. That is, (i) the driving force behind the escalation is none else but the invariant leading eigenspace of and large spectral gaps in self-attention matrices (while other operations in the Transformer block do not interfere with the escalation process); and (ii) the similarity measure converges to 1 at a global linear rate which eventually accelerates to the local rate of 1/2 under the standard setting.}
  
    \item 
    %\textcolor{blue}
    {Based on the insights gained from our analysis, we propose a simple de-escalation strategy to remove excessive similarity and restore expressivity in deep Transformer models.  Our preliminary experiments have confirmed the efficacy of the proposed strategy, substantially improving the training quality for very deep post-norm Transformers.  In contrary to existing techniques, our proposed method does not discount, explicitly or implicitly, the role of self-attention mechanism relative to other components.}

\end{itemize}

\subsection{Notation}

In general, matrices are denoted by upper-case letters and vectors by lower-case letters.
For any square matrix $M$, let $\lambda_i(M)$ be the $i$-th eigenvalue of $M$ arranged in a descending order in magnitude unless otherwise specified. 
We use $[W]_{ij}$ to denote the $ij$-th element of a matrix $W$, and will do so similarly for vector elements as well.  We denote $\E[\cdot] := \mathbb{E}_{[W]_{ij}\sim w}[\cdot]$. By default, the vector norm $\|\cdot\|$ is the Euclidean norm.
The symbol $\One$ denotes the vector of all ones in $\R^n$, and $e = \One/\sqrt{n}$.   We reiterate that capitalized Transformers are reserved for models based on the classic architecture, while later proposed modifications are referred to as transformers in small letters.
We will occasionally use the notation $\Pi_1\equiv \Po$ and $\Pi_2\equiv \Pop$.
For brevity, we will use the acronym TSE for Token Similarity Escalation.

%%%%%%%%%%%%%%%%%%%

\section{Analysis of TSE in Transformer}

In general, a Transformer model can comprise both an encoder and a decoder. In this paper, our focus will be solely on the encoder side of Transformer models.  Precisely speaking, our analysis is for Transformer encoders with random weights (which is the case at initialization).

%%%%%%%%%%%%%%%%%%%
\subsection{Transformer Architecture}

We will start with single-head Transformer layers and then show that our results extend to multi-head layers as well.  
An $L$-layer (encoder only) Transformer is the repeated composition of a layer function, say, $Y = \mathbf{postLN}(X)$, by $L$ times:
\begin{equation*}
Y = \mathbf{Transformer}(X) := 
\underbrace{\mathbf{postLN}\circ\cdots\circ\mathbf{postLN}}_{L ~\mbox{times}}\,(X).
\end{equation*}

We formalize the Transformer layer function $Y = \mathbf{postLN}(X)$ into Algorithm~\ref{alg:layer} below, which follows exactly the original architecture as proposed in \cite{vaswani2017attention}, and is organized into four steps for our treatment convenience.  
A few comments about the Transformer block in Algorithm~\ref{alg:layer} are in order.
\begin{algorithm2e}[ht!]	
\label{alg:layer}
\vspace{2mm}			
    \caption{$Y=\mathbf{postLN}(X)$}
    \KwIn{$X \in \R^{n \times d}$ (with weights $W, W_1, W_2$ and $\alpha > 0$).}
    \setcounter{AlgoLine}{-1}
    Compute a row-stochastic, attention matrix $P = P(X)\in \R^{n \times n}$.\\
    SA plus Residual: \hspace{5.5mm}$Y_1 =  PXW\alpha+X$.\\
    Layer Normalization: $Y_2 =  \LN(Y_1)$.\\
    FFN plus Residual:  \hspace{2.5mm}$Y_3 =  \phi(Y_2W_1)W_2+Y_2$.\\
    Layer Normalization: $Y_4 =  \LN(Y_3)$.\\
    \KwOut{$Y := Y_4\in \R^{n \times d}$} 
\end{algorithm2e}

\begin{itemize}
\item 
In Step~0, we allow the generality of utilizing different formulas for computing an attention matrix $P=P(X)$.  Besides $X$, such formulas also involve their own learnable weights that are not explicitly shown.  
\item 
The first term in Step~1 is called self-attention (SA) where the attention matrix $P(X)$ is applied to $X$ itself to form a nonlinear operation that also includes the multiplication by a weight matrix $W \in \R^{d \times d}$ from the right.   Next to $W$, a positive parameter $\alpha$ is introduced to balance the contributions from the SA mechanism relative to the residual (or skip connection) term $X$.  For weight matrices of a given norm, the smaller $\alpha$ is, the lesser role SA would play relative to residual. 
\item 
The first term in Step~3 is a feedforward network (FFN) with activation function $\phi(\cdot)$ and two weight matrices $W_1, W_2^T \in \R^{d \times q}$ where the column size $q$ is usually greater than $d$.  In our experiments, we will use the popular ReLU function for activation by default unless otherwise specified.
\item 
In Steps~2 and 4, layer-normalizations are applied row-wise to the input matrix.  More specifically, for any vector $x$, $\LN(x) = \gamma(x-\mu)/\sigma + \lambda$ where $\mu$ is the mean and $\sigma$ is the standard deviation of $x$, and $\gamma$ and $\lambda$ are learnable scalar parameters.
\end{itemize}

The impact of weight matrix initializations on model training has been widely studied since the use of an initialization scheme may determine the success or failure of training, see a recent survey~\cite{narkhede2022review} on this subject.  In our analysis and experiments, we will investigate the TSE phenomenon mostly under the widely used Xavier initialization~\cite{glorot2010understanding}.

%\textcolor{blue}
{In our experiments, we use the standard softmax formula~\cite{bridle1990probabilistic} to construct attention matrices $P$ (though we have tried other formulas without notable differences), which is defined as follows. For $X \in \R^{n\times d}$, compute $M = XW_q(XW_k)^T/\sqrt{d} \in \R^{n\times n}$ and let
\begin{equation}\label{softmax}
[P(X)]_{ij} := \exp\left( [M]_{ij}\right)/ \sum_{j=1}^n \exp \left([M]_{ij} \right).
\end{equation}
where $W_q$ and $W_k$ are weight matrices with elements drawn uniformly from $(-1,1)/\sqrt{d}$.}

%%%%%%%%%%%%%%%%%%%
\subsection{How Self-Attention Drives TSE}

In this subsection, we conduct a comprehensive analysis on how token similarity escalates after each time the self-attention step $Y = X + \alpha PXW$ is carried out.   We first motivate our analysis using a power-method-based intuition.

\subsubsection{An intuitive interpretation} 
The reason behind TSE can be intuitively explained by extending the idea of the power method for computing the largest eigenvector of a matrix.  
For any stochastic matrix $P$, the largest eigenvalue is always 1 corresponding to the leading eigenspace $\spanOne$.  From the convergence theory of power method, we know that if $|\lambda_2(P)|<1$ and $\Po x \ne 0$, then $P \cdots PPx$ converges to $\spanOne$ as the number of multiplications goes to infinity.  By the same argument,  for a sequence of stochastic matrices $\{P_k\}$ for which $\{|\lambda_2(P_k)|\}$ is bounded away from 1 (in fact a much weaker condition will suffice), then $P_k\cdots P_2 P_1 x$ will also converge to $\spanOne$ as $k$ goes to infinity.  Furthermore, for any $\alpha > 0$ there also holds that, as $k \to \infty$ and with some normalization sequence $\{c_k\}\subset \R$,
\[
c_k(I+\alpha P_k)\cdots(I+\alpha P_2)(I+\alpha P_1)x \to \spanOne,
\]
since the leading eigenspace is invariant under the above scaling and shifting operations. 
%\textcolor{blue}{However, it should be noted that there is limited literature available concerning the spectral gap of stochastic matrices. The author~\cite{wu2023asymptotic} delves into the empirical singular value distribution of doubly stochastic matrices.}
 
%
Now consider a sequence of SA operations (including the residual connection): 
\[
\mathrm{SA}_i(X) = X+\alpha P_iXW_i, \;\; i = 1, 2, 3, \cdots
\] 
each of which is associated with a weight matrix $W_i\in\R^{d\times d}$ and a stochastic matrix $P_i \equiv P_i(X)\in\R^{n\times n}$.  In addition to the SA operation, each Transformer block also utilizes other operations (feedforward network and layer normalizations). However, these other operations have no impact on TSE which is solely driven by the SA operation, as will be demonstrated later.  Specifically, the driving force is the spectral properties of $P_i$, while the fact that $P_i$ depends on $X$ is inconsequential.  Hence, roughly speaking, the escalation process can be characterized as
\begin{equation}\label{limit:SA_k}
c_k (\mathrm{SA}_k \circ \cdots \circ \mathrm{SA}_2 \circ \mathrm{SA}_1)(X) 
\to \textrm{span}\{\One e_j^T\}_{j=1}^d, 
\end{equation}
where $e_j\in\R^d$ is the unit vector with the $j$-th entry being unity, and $\{c_k\}\subset \R$ is a proper normalization sequence. In terms of the TSE behavior, each of the above $\mathrm{SA}_i$ operators can be viewed as a linear operator with a fixed leading eigenspace given in the left-hand side of \eqref{limit:SA_k}. 
  
%Moreover, the left-hand side of \eqref{limit:SA_k} can be rewritten as $c_k\Pi_{i=k}^1(I_{nd} + W_i^T \otimes P_i) \mathrm{vec}(X)$, where $I_{nd}$ is the identity in $\R^{nd}$, $\otimes$ denotes matrix Kronecker product and $\mathrm{vec}(X)$ is the vectorization of $X\in\R^{n\times d}$.  Under suitable conditions, it can be verified that the $d$-dimensional leading eigenspaces of $I + W_i^T \otimes P_i$ are indeed given by the vectorized form of the right-hand side of \eqref{limit:SA_k}.  

The above intuitive interpretation motivates us to conduct a comprehensive analysis on TSE.  To do so, we need to develop a rigorous approach to analyzing a precisely defined quantity that is critical for our TSE analysis.

\subsubsection{A theoretical analysis}
We start by stating some basic assumptions and facts.

\begin{assumption}\label{assumpt:0} 
In the self-attention formula $Y = X + \alpha PXW$, 
\begin{enumerate}
\item 
the matrix $W\in \R^{d \times d}$ is randomly initialized so that elements of $W$ are all independent with mean-zero and variance $\sigma^2$;
\item
the matrix $P$ is row-stochastic (or right-stochastic or Markov) so that $P \One=\One$ with the spectral radius equal to one (for example, see \cite{horn2012matrix}).  
\end{enumerate}
\end{assumption}

Under Assumption~\ref{assumpt:0}(1), there hold
\begin{equation}\label{W moments}
\E[W]=0 \in \R^{d\times d} ~~~~\mbox{ and }~~~~ \E[WW^T] = d\sigma^2I \in \R^{n\times n} .
\end{equation}

%In addition, without of loss of generality, we will use $\alpha=1$ most of the time (in effect to absorb it into the matrix $W$) unless specified otherwise.

We recall the definitions of token similarity, diversity and the relationship between them: 
\begin{equation*}
\Sim(X) = \|\Po X\|_F^2 / \|X\|_F^2,  \;\;\;
\Div(X)=\|\Pop X\|_F^2 / \|X\|_F^2,  \;\;\;
\Sim(X) + \Div(X)=1.
\end{equation*}
Token similarity of $X$ equals 1 implies the rank of $X$ being 1.  However, it is entirely possible that when token similarity is fairly close to 1, say $\Sim(X)=0.99$, $X$ is still numerically full-rank.  Therefore, using the continuous quantity $\Sim(X)$ to measure similarity is much more reasonable than using the discrete quantity rank as a measure of quality for representation (or embedding) matrices.  

We now introduce a critical quantity called TSE rate, or simply escalation rate.
\begin{definition}
For a pair of matrices $X, Y \in \R^{n\times d}$ so that $\Sim(X), \Sim(Y) \in (0,1)$, the escalation rate from $X$ to $Y$ is
\begin{equation}\label{def:r(X,Y)}
r(X,Y) := \frac{1-\Sim(X)}{1-\Sim(Y)} = \frac{\Div(X)}{\Div(Y)}.
\end{equation}
\end{definition}

Clearly, $r(X,Y)>1$ implies that $Y$ has a higher similarity (or a lower diversity) than $X$ does.  For $Y = X+\alpha PXW$, we aim at analyzing the expected value of $r(X,Y)$ with respect to the random matrix $W$ under Assumption~\ref{assumpt:0}.  The following proposition gives a key identity for the TSE rate $r(X,Y)$ that facilitates our analysis.

\begin{proposition}  \label{prop:r(X,Y)}
Given $X, Y \in \R^{n\times d}$ with $\Sim(X), \Sim(Y) \in (0,1)$, let $r(X,Y)$ be the escalation rate defined in \eqref{def:r(X,Y)}.  
Then the following identity holds
\begin{equation}\label{E3:rate}
r(X,Y) = 1+\left({\xi_1}/{\xi_2}-1\right)\Sim(X),
\end{equation}
where
\begin{equation}\label{def:xi's}
\xi_i \equiv \xi_i(X,Y) := \frac{\|\Pi_i Y\|_F^2}{\|\Pi_i X\|_F^2}, \;\; i = 1,2.  
%\;\;\; \mbox{ ~and~ } \;\;\;  \xi_2 \equiv \xi_2(X,Y) := \frac{\|\Pop Y\|_F^2}{\|\Pop X\|_F^2}.
\end{equation}
Therefore, $r(X,Y) > 1$ if and only if $\xi_1>\xi_2$.
\end{proposition}
\begin{proof}
We first note that $\|Y\|_F^2 / \|X\|_F^2 = \xi_1\Sim(X)+\xi_2\Div(X)$.
%\[
%\|Y\|_F^2 / \|X\|_F^2 = \|\Pop Y\|_F^2/ \|X\|_F^2 + \|\Po Y\|_F^2/ \|X\|_F^2 = \xi_2\Div(X) + \xi_1\Sim(X). 
%\]
Then by direct calculations,
\begin{equation*}\label{ratio:Div in Apdx}
\frac{\Div(X)}{\Div(Y)} 
= \frac{\|Y\|_F^2 / \|X\|_F^2} {\|\Pop Y\|_F^2 / \|\Pop X\|_F^2}
= \frac{\xi_2(1-\Sim(X)) + \xi_1\Sim(X)} {\xi_2}
= 1+(\xi_1/\xi_2-1)\Sim(X).
\end{equation*}

The second statement is obvious.
\end{proof}

Proposition~\ref{prop:r(X,Y)} indicates that as long as $\xi_{1}>\xi_{2}$,  token similarity of $Y$ will be larger than that of $X$; in other words, there happens an escalation in token similarity from $X$ to $Y$. For the SA-plus-residual step $Y=X+\alpha PXW$, we aim to analyze the expected escalation rate with respect to $W$, that is,
\begin{equation}\label{eq:E[r]}
 \E[r(X,Y)] = 1+\left(\E\left[{\xi_1}/{\xi_2}\right]-1\right)\Sim(X).
\end{equation}

The roadmap of our analysis is as follows. In order to estimate $\E[\xi_{1}/\xi_{2}]$, which in general does not allow a closed-form formula, we show instead that under mild conditions $\xi_{1}/\xi_{2}$ is highly concentrated at $\E[\xi_{1}]/\E[\xi_{2}]$.  Then it will suffice to estimate $\E[\xi_{1}/\xi_{2}]$ through the two individual expected values, $\E[\xi_{1}]$ and $\E[\xi_{2}]$, which are calculated in the lemma below.

%%%%%%%%%%%%%%%%%
\begin{lemma}\label{lem:E(xi)}
Given $X \in \R^{n\times d}, P \in \R^{n\times n}, W \in \R^{d\times d}$ and $\alpha>0$, let $Y=X+\alpha PXW$. Under Assumption~\ref{assumpt:0}, the expected values of $\xi_i$, $i=1,2$, with respect to $W$ are, respectively,
\begin{equation} \label{E2:xi1-2}
   \E[\xi_i] = 1 + \alpha^2d\sigma^2 \mu_i^2, \;\; i = 1, 2,
\end{equation}
where $\xi_i$ are defined in~\eqref{def:xi's}, and
\begin{equation}\label{def:mu's}
\mu_i \equiv \mu_i(X,P) := \frac{\|\Pi_i PX\|_F}{\|\Pi_i X\|_F}, \;\; i=1,2.   
%\mu_2 \equiv \mu_2(X,P) := \frac{\|\Pop PX\|_F}{\|\Pop X\|_F}.
\end{equation}
Furthermore, there holds the bounds for $\mu_1$ and $\mu_2$, 
\begin{equation}\label{mu:bounds}
\mu_1^2 \ge (1-\omega)^2,  \;\;\;\; \mu_2^2 \le \delta^2,
\end{equation}
where 
\begin{equation}\label{def:omega-delta}
\omega = \|e^TP\Pop X\|/\|e^T X\| ~~~\mbox{ and }~~~ \delta = \|\Pop P\|_2.
\end{equation}
%In particular, when either $P$ is doubly stochastic or $X=eu^T$ for some $u\in\R^d$, then %$\mu_1=1$  and $\omega=0$.
%\[ \mu_1-1 = \mu_2 = \omega = 0. \]
\end{lemma}
A proof for this lemma will be presented in the Appendix.

%%%%%%%%%%%%%%%%%%
\begin{proposition} \label{thm:E-factor}
Given $X \in \R^{n\times d}, P \in \R^{n\times n}, W \in \R^{d\times d}$ and $\alpha>0$, let $Y=X+\alpha PXW$.  Under Assumption~\ref{assumpt:0}, there holds
\begin{equation} \label{eq1:E(r)}
    \E[r(X,Y)] = 1+ \left(\frac{\alpha^2d\sigma^2}{1+\alpha^2d\sigma^2\mu_2^2}
    (\mu_1^2-\mu_2^2)-\E[\eta] \right)\Sim(X),
\end{equation}
where
\begin{equation}\label{def:eta}
\eta := \E [\xi_1] / \E [\xi_2] - \xi_1/\xi_2,
\end{equation}
\end{proposition} 
The equality \eqref{eq1:E(r)} follows from replacing $\E[\xi_1/\xi_2]$ by $\E[\xi_1]/\E[\xi_2]-\eta$ in \eqref{eq:E[r]}, and then substituting in the expressions for $\E[\xi_1]$ and $\E[\xi_2]$ given in \eqref{E2:xi1-2}.  

We will use some concentration inequalities to show that under suitable conditions the quantity  $\E[\eta]$ will be amply small so that TSE occurs with overwhelming probability. Our main theoretical result is given as the following theorem.  To state the result, we first note from \eqref{mu:bounds} that the following condition always holds
\begin{equation}\label{bnd:tech}
\max \left\{\mu_1^2-(1-\omega)^2,\delta^2-\mu_2^2\right\} \ge 0,
\end{equation}
where the left-hand side depends on $X$ and $P$.

%%%%%%%%%%%%%%%%%%%%%%%%%%%
\begin{theorem} \label{thm:r-rate}
Given $X \in \R^{n\times d}, P \in \R^{n\times n}, W \in \R^{d\times d}$ and $\alpha>0$, let $Y=X+\alpha PXW$. 
Assume that the elements of $W \in \R^{n\times d}$ are independent, mean-zero, sub-gaussian random variables 
with sub-gaussian norm equal to $\sigma = 1/\sqrt{d}$.  
Suppose that $\omega + \delta < 1$ and the left-hand side of \eqref{bnd:tech} is bounded away from zero for $d$ sufficiently large.  
Then the expected escalation rate satisfies
\begin{equation}  \label{E5:rate}
    \E[r(X,Y)] \geq 1+ \frac{\alpha^2}{1+\alpha^2\delta^2}\left((1-\omega)^2-\delta^2\right)\Sim(X)>1.
\end{equation}
\end{theorem}
The proof for Theorem~\ref{thm:r-rate} will be again given in the Appendix.  

The most critical quantities in the bound~\eqref{E5:rate} are $\omega$ and $\delta$.  Whenever these two numbers are small, TSE happens; and the smaller they are, the faster is the escalation rate. The following corollary considers a few special situations for the attention matrix $P$, indicating that spectral properties, especially spectral gaps, of $P$ play a key role in the TSE process.  This  corollary can be directly verified from the definitions of the involved quantities.
\begin{corollary}
Let the conditions in Theorem~\ref{thm:r-rate} hold.
When $P$ is doubly stochastic, then $\mu_1= 1$ and $\omega = 0$.  In this case, for $\alpha=1$ the expected escalation rate estimate in \eqref{E5:rate} reduces to
\begin{equation*}\label{if1:P'e=e}
\E[r(X,Y)] \geq 1+ \frac{1-\delta^2}{1+\delta^2} \; \Sim(X).
\end{equation*}
Moreover, when $P$ is symmetric (thus doubly stochastic), 
then $\delta^2 =  |\lambda_2(P)|^2$ and
\begin{equation}\label{if2:P=P'}
\E[r(X,Y)] \geq 1+ \frac{1-|\lambda_2(P)|^2}{1+|\lambda_2(P)|^2}\; \Sim(X),
\end{equation}
where $\lambda_2(P)$ is the second largest eigenvalue of $P$ in modulus, and $1-|\lambda_2(P)|^2$ serves as a measure of the spectral gap of $P$.
\end{corollary}
As we will empirically demonstrate later, in practice formula~\eqref{if2:P=P'} turns out to be a better estimate for $\E[r(X,Y)]$ than the guaranteed lower bound in \eqref{E5:rate} which can be overly conservative.  It is worth noting that as $\Sim(X)$ becomes close to 1, the corresponding attention matrices $P(X)$ computed from the softmax formula \eqref{softmax} (and those from most other formulas as well) will be close to $\One\One^T/n$, which is symmetric with $\lambda_2(P)=0$.

Aside from $\Sim(X)$, the estimation formula in \eqref{if2:P=P'} is entirely determined by the spectral gap of $P$.  If $|\lambda_2(P)|=1$ (say, for $P=I$), then no escalation would happen.  But if $|\lambda_2(P)| \ll 1$, the escalation rate can be large, which in fact is what happens in reality for random-like row-stochastic matrices $P$ even though $P$ is asymmetric in general. This phenomenon not only happens for $P$ computed by the softmax formula \eqref{softmax}, but also for $P$ computed by other formulas as well.

\begin{remark}
Some comments are due for Theorem~\ref{thm:r-rate} concerning the acceleration of TSE.
\begin{itemize}
\item 
The term $\Sim(X)$ in \eqref{E5:rate} or \eqref{if2:P=P'} plays a role of accelerator for TSE.  That is, the larger it is, the faster is the escalation since the factor in front of $\Sim(X)$ does not vary significantly.  As $\Sim(X)$ goes to 1,  both $\omega$ and $\delta$ converge to 0 so that the estimated escalation rate on the right-hand side of \eqref{E5:rate} approaches its maximum at $1+ \alpha^2$.  
\item
If we consider the convergence of similarity to 1 (or diversity to 0), then the rate is linear and asymptotically $1/(1+ \alpha^2)=1/2$ for $\alpha=1$.  Experiments show that the rate 1/2 occurs quite early in practice, resulting in a fast linear convergence.
\end{itemize}
\end{remark}

%%%%%%%%%%%%%%%%%%%%%%%%%%%%%%%%%%%
Next, we show that the above TSE result can be straightforwardly extended to the case of multi-head self-attention.  For this purpose, it suffices to show that in the multi-head case, the expected escalation rate $\E[r(X,Y)]$ has an identical expression as in \eqref{eq1:E(r)} for the single-head case except that in the former case, relevant quantities are the average over multiple heads.  Now assume that the number of heads is $h$ and $d$ is divisible by $h$.

\begin{proposition}\label{prop:multihead}
Given $X \in \R^{n\times d}$, $P_k \in \R^{n\times n}$ and $W_k \in \R^{d \times d/h}$, for $k=1,\cdots,h$, and $\alpha>0$, let
\[
Y = X+ \alpha[P_1XW_1\;\; P_2XW_2\;\; \cdots \;\; P_hXW_h].
\]
Under Assumption~\ref{assumpt:0} (with column number $d$ for $W$ changed to $d/h$ for all $W_k$),
\begin{equation} \label{eq2:E(r)}
    \E[r(X,Y)] = 1+ \left(\frac{\alpha^2d\sigma^2}{1+\alpha^2d\sigma^2\bar{\mu}_2^2}
    ( \bar{\mu}_1^2- \bar{\mu}_2^2)-\E[\eta] \right)\Sim(X),
\end{equation}
where $\eta$ is defined as in \eqref{def:eta} and
\begin{equation*}
\bar{\mu}_i^2 := \frac{1}{h}\sum_{k=1}^h\frac{\|\Pi_i P_k X\|_F^2}{\|\Pi_i X\|_F^2}, \;\; i = 1,2.
\end{equation*}
\end{proposition} 
The proof for this proposition is given in the Appendix. 

Compared to \eqref{eq1:E(r)} in Proposition~\ref{thm:E-factor}, we see that the only difference in \eqref{eq2:E(r)} is that $\mu_i^2$ are replaced by $\bar{\mu}_i^2$ which are the average across multiple heads.  Additionally, the bounds in \eqref{mu:bounds} for ${\mu}_i^2$ also hold for their average counterparts $\bar{\mu}_i^2$ once we replace $(1-\omega)^2$ and $\delta^2$ by their corresponding average counterparts, that is, $\bar{\mu_1}^2 \ge \frac{1}{h}\sum_{k=1}^h(1-\omega_k)^2$ and $\bar{\mu_2}^2 \le \frac{1}{h}\sum_{k=1}^h \delta_k^2$ where quantities with $k$-indices are still defined as in \eqref{def:omega-delta} except now associated with different $P_k$'s.

Equipped with Proposition~\ref{prop:multihead} and following the same line of arguments, we can readily derive the counterpart of Theorem~\ref{thm:r-rate} and thus extend our TSE analysis from the single-head case to multiple-head cases.

%%%%%%%%%%%%%%%%%%%
\subsection{Other Steps Do Not Impact TSE} 
%%%%%%%%%%%%%%%%%%%

We first examine the FFN Step in Algorithm~\ref{alg:layer}.  We are not aware of any previous report that the classic FFN architecture has any involvement with TSE under any circumstances.   This is easy to explain from the following column-space viewpoint.  Essentially, TSE implies that the column space of $X$ is moving towards the subspace $\spanOne$. However, in FFN, the right-multiplications of $X$ by weight matrices do not change, in one way or another, the column space of $X$ at all.  For example, if $\Pop X$ is small relative to $\Po X$, then so should be $(\Pop X)W$ relative to $(\Po X)W$ for any generic $W$. Furthermore, no discernible reason is seen to expect any impact on TSE from the usual element-wise activation functions such as ReLU.  On the contrary, it is evident that the subspace $\spanOne$ is invariant under any element-wise activation function $\phi: \R\rightarrow\R$.  Therefore, such functions actually preserve similarity of $X$ and do not interfere with the TSE process.

Next we examine the impact of layer normalizations on TSE.  The layer normalization function $Y=\LN (X)\in\R^{n\times d}$ can be expressed (without scaling and shifting) as
\begin{equation*}
Y = DX\left(I-\One_d\One_d^T/d\right),
\end{equation*}
where $D$ is a diagonal matrix with $[D]_{ii} = 1/\sigma_i$, and $\sigma_i^2$ is the variance of the $i$-th row of $X$.  Clearly, if $X=\One v^T$, then its column space $\spanOne$ is invariant under layer normalizations since the $D$-matrix is a multiple of identity.  This indicates that layer normalizations do not de-escalate high token similarity. Moreover, if the elements of $X$ are independently random with a fixed variance, then the corresponding $D$-matrix will also be close to a multiple of identity, thus preserving the column space of $X$.
In either case, the TSE process is not interfered with one way or another by layer normalizations.

%%%%%%%%%%%%%%%%%%%
\subsection{Experimental Verification}

In this subsection, we provide some strong empirical evidence to corroborate our theoretical results.   Our experiments are carried out on multi-head, Vision Transformers\footnote{https://github.com/lucidrains/vit-pytorch} that follows the block architecture given in Algorithm 1.  We observe some key quantities at each block at the initial state where the input matrix $X\in\R^{n\times d}$ is randomly drawn from the standard normal distribution $\mathcal{N}(0, 1)$ and the weight matrices $W_k$ are randomly initialized from $\mathcal{N}\left(0,\frac{h}{d}I\right)$ (while other weight matrices are initialized using the default method in the code).  The model parameters are set to $n=64$, $d=512$, $h=8$, $\alpha=1$ and the depth is set to 20.  We always calculate attention matrices using the softmax formula \eqref{softmax}(though preliminary trials suggest that other formulas would essentially give the same results).  We mention that the above experimentation setting, with different depth values, will be again used in the next section.

In the first experiment, we run 50 independent random trials.  The average values of several important quantities are presented in Figure~\ref{fig:them-unif}.
\begin{figure}[htp]
\centering
\includegraphics[width=.8\textwidth,trim=0 0 0 0, clip]{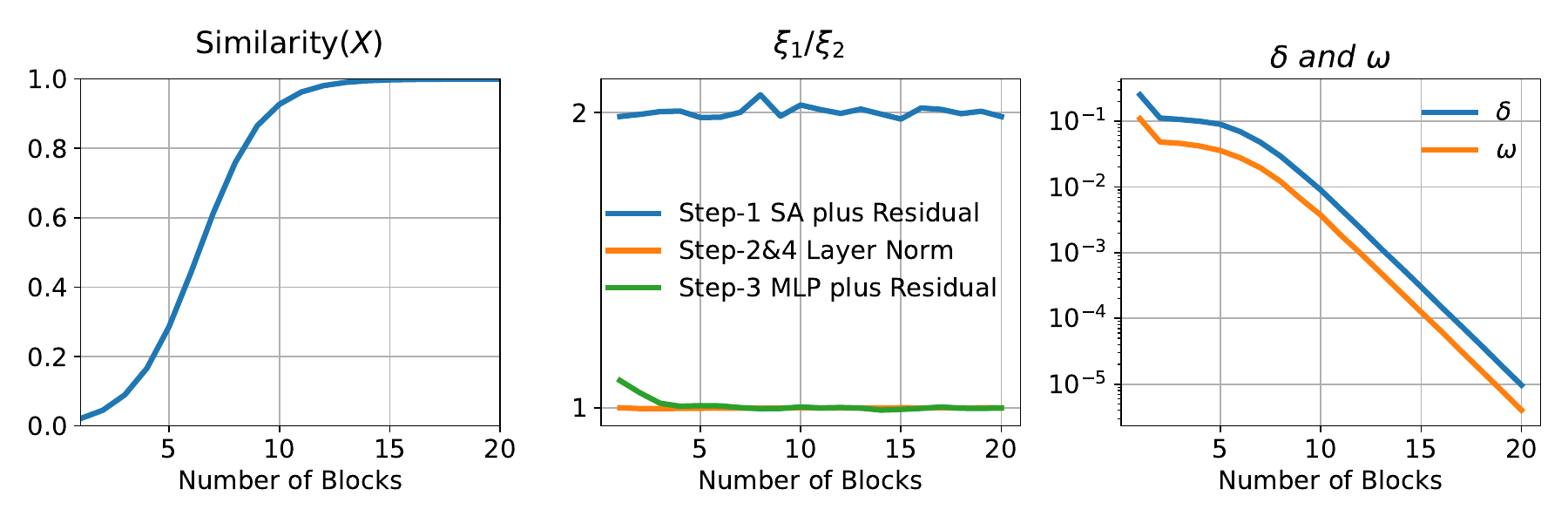}
\caption{Average values of token similarity, $\xi_1/\xi_2$, $\delta$ and $\omega$ over 50 trials.  
At each multi-head block, $\delta$ and $\omega$ are also averaged across the eight heads.
}
\label{fig:them-unif}
\end{figure}

From the left plot in Figure~\ref{fig:them-unif}, we observe that similarity monotonically increases throughout all blocks, approaching the maximum before block 15, due to the fact that in Step 1 the (sampled) mean of $\xi_1/\xi_2$ stays around 2 from the very beginning, while the other three steps have a negligible impact on token similarity (since $\xi_1/\xi_2\approx 1$), as can be seen in the middle plot.  Moreover, we see from the right plot that the quantities $\delta$ and $\omega$ are far less than 1 from the start and quickly approach 0 as the depth increases.

\begin{figure}[htp]
\centering 
\includegraphics[width=.8\textwidth,trim=0 0 0 0, clip]{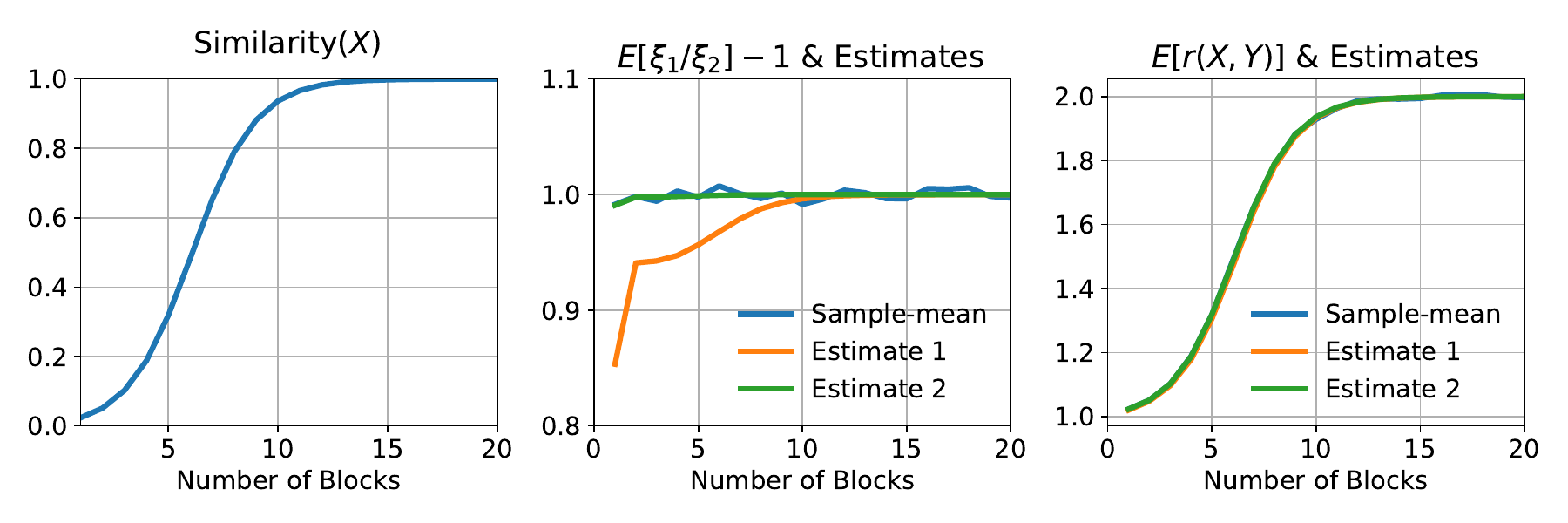}
\caption{Average values over 1000 trials of token similarity, $\xi_1/\xi_2-1$ and $r(X,Y)$ and 
their estimates from \eqref{E5:rate} and \eqref{if2:P=P'} for the latter two quantities
}
\label{fig:estimates} 
\end{figure}

In the second experiment, we only randomize the weight matrix $W$ in Step~1 and run 1000 trials, while fixing all other quantities including $X$.  This is exactly the same setting under which we derived expected values for relevant quantities in our analysis.   
In Figure~\ref{fig:estimates}, the left plot gives the average token similarity of $X$ which looks visibly identical to the one given in the left plot of Figure~\ref{fig:them-unif}.

In the middle plot of Figure~\ref{fig:them-unif}, we present the sample mean of $\xi_1/\xi_2-1$ and the two estimates given in \eqref{E5:rate} (with $\alpha=1$) and \eqref{if2:P=P'}.  Namely, Estimate~1 and 2 are, respectively, 
\[
\frac{(1 -\omega)^2-\delta^2}{1+\delta^2} \;\;\mbox{ and }\;\;
\frac{1-|\lambda_2(P)|^2}{1+|\lambda_2(P)|^2}.
\]
It is interesting to observe that although Estimate~2 is theoretically valid only for symmetric attention matrix $P$, in the experiment it actually provides a closer approximation to $\E[\xi_1/\xi_2]-1$ than the guaranteed lower bound in Estimate~1.  This experiment exemplifies the argument that large spectral gaps of attention matrices are one of the real driving forces behind TSE, considering that Estimate~2 only depends on $\lambda_2(P)$.

In the right plot of Figure~\ref{fig:them-unif}, we present the sample mean of $r(X,Y)$ and the corresponding two estimates.  As we can see, the three curves are nearly identical.  This is due to the fact that, in either \eqref{E5:rate} or \eqref{if2:P=P'}, the token similarity term dominates the factor in front of it which varies relatively mildly.

%%%%%%%%%%%%%%%%%%%
\subsection{Discussion}

Our analysis reveals that the driving force behind token similarity escalation is two-fold: 1) the existence of the invariant leading eigenspace, $\spanOne$, for all attention matrices which are stochastic (or Markov), and 2) large spectral gaps commonly present in computed attention matrices.  

Indeed, it has been established in \cite{bordenave2012circular} that, under mild conditions, $n$ by $n$ stochastic (or Markov) matrices generated from normalizations of i.i.d.~nonnegative random variables almost surely have large spectral gaps of the order $1-O(1/\sqrt{n})$ (see also the earlier work~\cite{chafai2010dirichlet}).  With randomly initialized weights in their construction, attention matrices are random (and stochastic) to also possess large spectral gaps with high probability, even though they may not necessarily or strictly satisfy all the required theoretical assumptions such i.i.d.~randomness.

%Such large spectral gaps are likely an innate property of random-like attention matrices though so far we have not found a fitting result in this regard (even for i.i.d.~random, row-stochastic matrices).

Our similarity measure $\Sim(\cdot)$ converges to 1 at a global linear rate while it itself also helps accelerate the rate. The asymptotical rate of convergence reaches 1/2 when $\alpha=1$. In view of the fact that $0.5^{10} \le 10^{-3}$,  this fast linear convergence explains why, under standard initializations, classic Transformers start to show some instability once the number of layers exceeds ten or so.

%%%%%%%%%%%%%%%%%%%%%%
\section{Mitigation of TSE in Transformers}

The analysis in the previous section underscores the issue of growing token similarity in post-norm Transformers. Many methods have been proposed to mitigate this problem, as discussed in our Related Work section.  Implicitly or explicitly, these methods invariably lead to reducing the role of self-attention relative to residual.  To take  a simplistic view, in the SA-plus-residual step $X + \alpha P(X)XW$ one could explicitly diminish the size of $\alpha$ to slow down the progress of TSE; or one could instead modify the step into $X + P(\hat{X})\hat{X}W$ where $\hat{X}$ is related to $X$ but has a smaller norm that $X$, thus implicitly reducing the role of self-attention.

%%%%%%%%%%%%%%%%%%
\subsection{Implicit Mitigation in Pre-norm}
More recently, large-scale transformer models, such as those~\cite{touvron2023llama, brown2020language}, opt for the pre-norm architecture. 
A pre-norm transformer block can be written as $Z = \mathbf{PreLN}(X)$ where $Z$ is computed as follows:
\begin{equation}\label{def:pre-norm}
\hat{X} = \LN(X), \;\;\; Y= X + P(\hat{X})\hat{X}W, \;\;\; Z = Y + \phi(\LN(Y)W_1)W_2,
\end{equation}
From the above formulas with the usual random initializations, it is straightforward to verify that the expected value of $\|Y\|_F^2$ is larger than $\|X\|_F^2$, by following the same proof technique in the proof of Lemma~\ref{lem:E(xi)}.  Specifically, under the standard weight initializations, for the above pre-norm transformer block there hold
\begin{equation}\label{pre-norm-sizes}
    \E[\|Y\|_F^2] = \|X\|_F^2 + \| P(\hat{X})\hat{X} \|_F^2 > \|X\|_F^2, ~~
    \E[\|Z\|_F^2] > \|Y\|_F^2.
\end{equation}
According to \eqref{pre-norm-sizes}, the output norm of the pre-norm block, $\|\mathbf{PreLN}(\cdot)\|_F$, is monotonically increasing in expectation, as is illustrated empirically in the right plot of Figure~\ref{fig:pre}.  Meanwhile, the norm of the self-attention input is always fixed at $\|\LN(\cdot)\|_F$ due to layer normalizations.  Consequently, in the pre-norm architecture the role of the self-attention mechanism is progressively diminishing with as the depth grows.
\begin{figure}[htp]
\centering
\includegraphics[width=.5\textwidth,trim=0 0 0 0, clip]{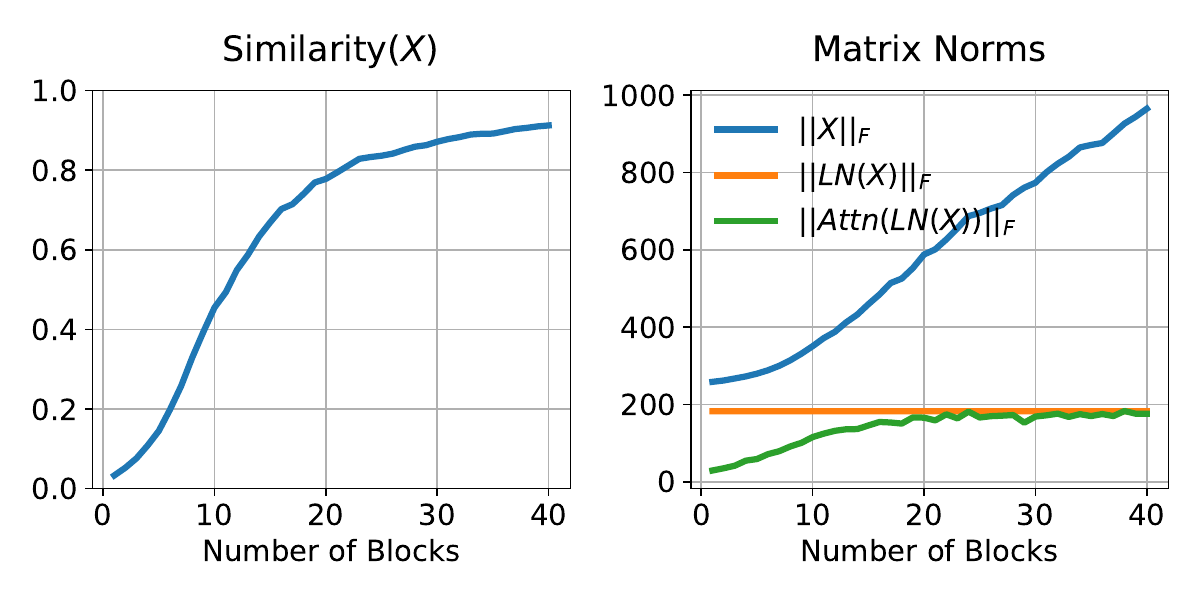}
\caption{
Left: Token similarity in a pre-norm transformer model. Right: Frobenius norms of input $X$~(blue), $\hat{X}$~(orange) and $P(\hat{X})\hat{X}W$~(green), see \eqref{def:pre-norm} for definitions, where attention matrices are computed by the softmax formula.}
\label{fig:pre}
\end{figure}

Nevertheless, as can be seen from the left plot of Figure ~\ref{fig:pre}, in the pre-norm model token similarity still keeps increasing, albeit at a much slower escalation rate than in post-norm models for the reason that we just explained.

%%%%%%%%%%%%%%%%%%%%%%%
\subsection{A Simple De-escalation Strategy}

We propose a simple strategy to counter the escalation of token similarity in deep Transformers, that is, to de-escalate the progressive growth in the subspace $ \textrm{span}\{\One e_j^T\}_{j=1}^d$.  Specifically, we will insert into Algorithm~\ref{alg:layer} a de-escalation step of the form
\begin{equation}\label{def:removal}
Y = \left(I - \tau\Po\right)X, \;\; \tau \in(0,1].
\end{equation}
Alternatively speaking, we first project $X$ onto  $\textrm{span}\{\One e_j^T\}_{j=1}^d$ by applying the projection to all columns of $X$, and then subtract a portion of the projection from $X$.  
The parameter $\tau\in(0,1]$ determines the removed portion. Particularly, $\tau=0$ and $\tau=1$ correspond to, respectively,  no de-escalation and a complete de-escalation.  
Additionally, when $\tau=1$ the operation becomes $Y=\Pop{X}$ which is equivalent to the centralization of the columns of $X$, i.e., subtracting the column mean from each column.
In principle, $\tau$ can be made a learnable parameter of the model, but the experimental results  reported here are all obtained using $\tau=1$.
There are several possible locations in Algorithm~\ref{alg:layer} to insert the proposed de-escalation step \eqref{def:removal}.  It remains to be determined whether different locations can make meaningful differences in the performance of large-scale models.  

In Figure~\ref{fig:spectra-ours}, we examine how the de-escalation operation, applied to the output of each layer (or the input of the next layer), affect the dynamics of token similarity in a base-line Transformer.  We observe that the levels of token diversity (one minus similarity) corresponding to 3 $\tau$-values,  $\tau \in \{ 0.1, 0.5, 1\}$, are around 0, 0.5 and 1, respectively.   Such a behavior seems quite predictable.  In particular, the phenomenon of TSE seems to have already been eliminated in the case of $\tau = 0.5$ (let alone the case of $\tau=1$).
\begin{figure}[ht]
\begin{center}
    \vspace{-.1cm}		
    \includegraphics[width=.5\textwidth,trim=0 0 0 0, clip]{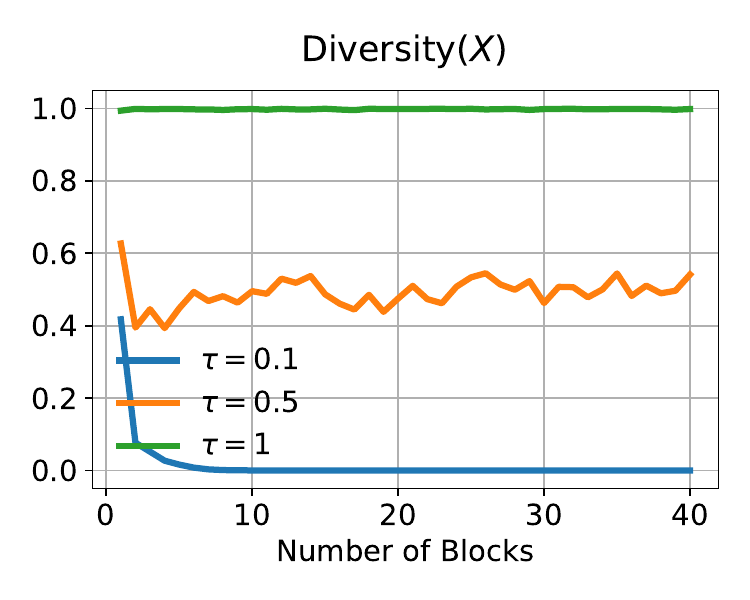}
    \vspace{-.1cm}
    \caption{
    The averaged values of 20 runs about Token Diversity in  De-escalated Transformers with softmax attention, where each block is the same as Algorithm 1. There are three de-escalation values: $\tau$ = 0.1, 0.5 and 1. }
    \label{fig:spectra-ours}
    \end{center}
    %\vspace{-.5cm}
\end{figure}

%%%%%%%%%%%%%%%%%%%%%%
%\subsection{Training Performance Evaluation}

In this following, we conduct preliminary numerical experiments to assess how the proposed de-escalation strategy \eqref{def:removal} influences the training behavior of deep, post-norm Transformers in comparison to their pre-norm counterparts.  As such, on each test instance, we always compare three model variants: post-norm model, pre-norm model, and ours where the de-escalation \eqref{def:removal} is added within each Transformer block with $\tau=1$ fixed. 
It should be cautioned that an ``optimal value" of $\tau$ may vary across different models and cases, thus calling for case-to-case investigations.  In the presented experiments, we do not fine-tune $\tau$-values as our primary purpose here is to demonstrate how the proposed scheme can significantly and positively impact the training of some deep transformer models.

Two relatively small datasets, CIFAR10~\cite{krizhevsky2009learning} for image classification and WikiText-103~\cite{zhai2023stabilizing} for natural language processing, will be used for our experiments due to constraints on available computing resources.
All the experiments are carried out using PyTorch~\cite{paszke2019pytorch} running on one Nvidia-V100 GPU. We emphasize that for each test instance we always run multiple trials, starting from different random initializations for model weights.  All the reported values are the average of at least 3 trials. 

%%%%%%%%%%%%%%%%%%%%%%
\subsection{Vision Transformer on CIFAR10}

We apply the Vision Transformer model ViT~\cite{dosovitskiy2020image} to the CIFAR10 dataset with the prescribed three model variants.  In this test, the de-escalation step \eqref{def:removal} is added to the end of each post-norm block. We utilize the widely used optimizer
AdamW~\cite{loshchilov2018decoupled} % including~\cite{he2022masked, radford2023robust}.
in which the hyper-parameters $\beta_1$ and $\beta_2$ are set to 0.9 and 0.999, respectively, along with a weight decay value of 0.1. A multi-step scheduler is employed with  reduction factor of $1/5$ at the 70\% and 90\% junctures of the training duration.  Image patches are configured to be 4 by 4, and the gradient-sampling batch size is set to 128.  Auto-augmentation~\cite{cubuk2020randaugment} is enabled for the dataset. Further details about this ViT model's configurations are given in Table~\ref{tab:settingsinvit}. With the depth 80, this tested model is qualified to be a deep transformer.   
\begin{table}[ht]
    \centering
    \caption{Model size parameters for ViT}
    \begin{tabular}{cccccc}
    \hline
        Depth & Hidden size & FFN size & Heads & Head size&  \\ \hline
      % ViT1  & 64 & 128 & 4 & 16  \\
      % ViT2  & 128 & 256 & 4 & 32  \\
       80 & 192& 384 & 8 & 24  \\ \hline
    \end{tabular}
%    \caption{Details in one Vision Transformer block.}
    \label{tab:settingsinvit}
\end{table}

The above settings are applied to all three model variants.  There does exist a difference in the choice of learning rate $lr$.  We used $lr = 10^{-4}$ for training the pre-norm and our models, while a smaller $lr = 0.5*\!10^{-4}$ was used for training the post-norm model in order to obtain a meaningful reduction in the loss value.

\begin{figure}[ht]
	\begin{center}
		\vspace{-.1cm}		
        \includegraphics[width=.8\textwidth,trim=0 0 0 0, clip]{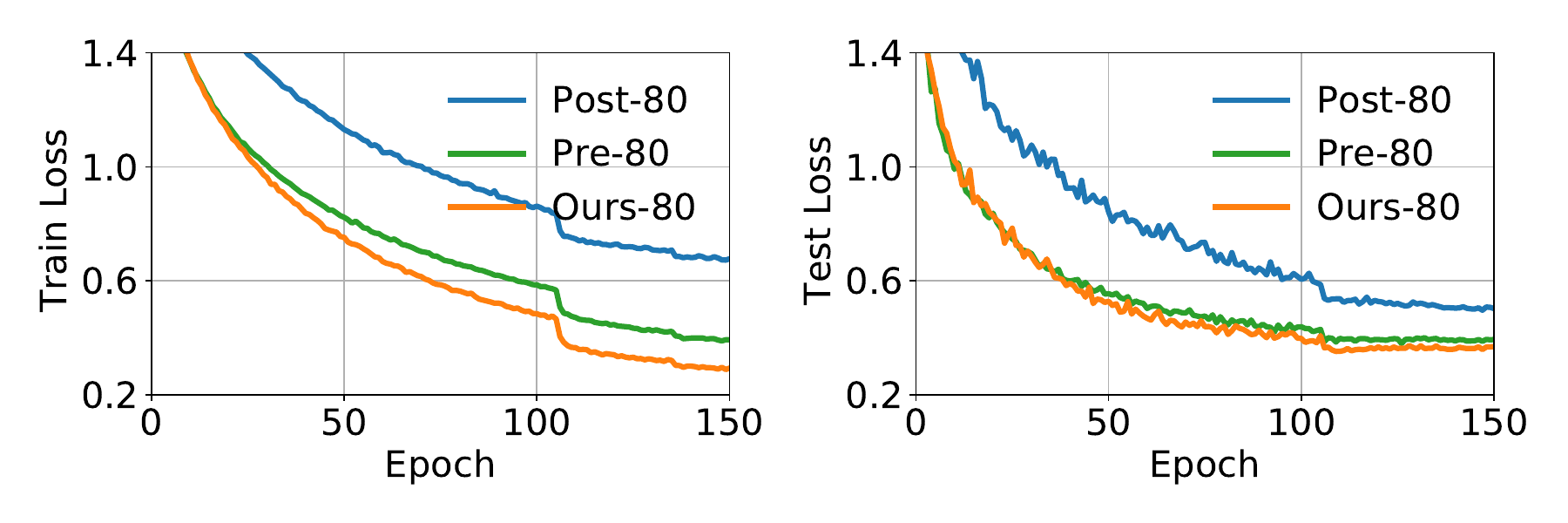}
		\vspace{-.4cm}
		\caption{Histories of training and testing losses of three ViT model variants on the CIFAR10 dataset with auto-augmentation.}
		\label{fig:lossvsepoch}
	\end{center}
	\vspace{-.1cm}
\end{figure}

Training histories, over the course of 150 epochs, of the three ViT model variants on CIFAR10 are presented in Figure~\ref{fig:lossvsepoch}. From this figure, we can see a significant performance gap between the pure post-norm model and our de-escalated post-norm model.  In this particular case, our de-escalated post-norm model also slightly outperformed its pre-norm counterpart in terms of training loss.  Overall, the effectiveness of our de-escalation strategy is unmistakable in this test.

\subsection{Transformer-XL on WikiText-103}
We next employ the Transformer-XL model~\cite{dai2019transformer} to perform experiments on  WikiText-103 dataset.  In order to create a deep transformer model of a manageable size, we made the following modifications: increasing the model depth from 16 to 60 and decreasing the FFN hidden width (i.e., the column number of $W_1$ and $W_2^T$ in Step~3 of Algorithm~1) from 2100 to 820.  The resulting deep model has about 201M model parameters which is moderately larger than the original size of 151M.  To train this larger model, we decrease the batch size from 60 to 40.  In addition, we disable the dropout and gradient clipping options to have a more generic optimization process. 
Beside the afore-mentioned modifications, we keep all the model hyper-parameters and optimization settings intact exactly as specified in \cite{dai2019transformer}, which we refer to for further details.

\begin{figure}[ht]
	\begin{center}
		\vspace{-.1cm}		
        \includegraphics[width=.4\textwidth,trim=0 0 0 0, clip]{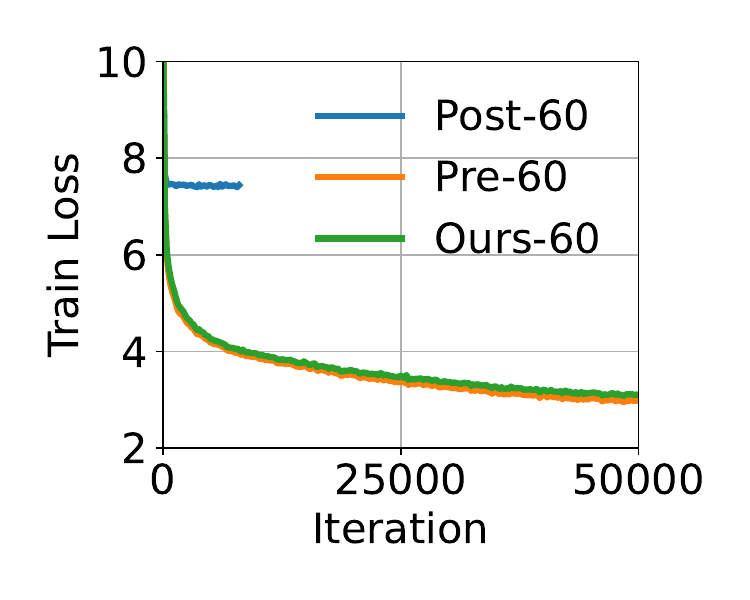}
		\vspace{-.5cm}
		\caption{Training loss histories of 3 Transformer-XL variants on 
		WiKiText103 dataset over 50000 iterations.}
		\label{fig:wt103}
	\end{center}
	\vspace{-.1cm}
\end{figure}

Figure~\ref{fig:wt103} gives the training performances of three Transformer-XL model variants of depth 60:  post-norm, pre-norm, and our de-escalated post-norm model in which we add de-escalation operation \eqref{def:removal} to the input of each FFN block in the post-norm model.  As we can see, the results are consistent with the previous experiment.  There is a huge performance gap between the post-norm and our de-escalated post-norm models (the former essentially failed, thus was cut short); while the pre-norm model and our de-escalated post-norm model performed similarly.

%%%%%%%%%%%%%% 
\subsection{Discussion}
Our proposed framework offers flexibility in placing the de-escalation operation at different locations within the Transformer block, determining the value of de-escalation strength $\tau$, and deciding whether de-escalation should be learnable or not.  During our experimentation, we did try a limited number of combinations for these settings and found that our proposed de-escalation strategy is generally effective in addressing TSE-related issues.  That said, real-world tasks are bound to be more intricate than our preliminary experiments, for which adjustments and calibrations of all available settings are most likely needed to deal with different problem scenarios and characteristics.

%%%%%%%%%%%%%% 
\section{Conclusion}
We conduct a comprehensive analysis on the phenomenon of token similarity escalation (TSE) in classic Transformers which leads to loss of expressive power in deep models.  Our theory reveals why and how TSE phenomenon occurs and what is the speed of escalation.  Based on insights gained from our analysis, we propose a simple and linear de-escalation operation to surgically remove excessive similarity from token representations without necessarily suppressing the role of self-attention.  Preliminary experiments show that, on moderate-scale transformers, the proposed de-escalation strategy effectively enables deep post-norm models to become competitive with their pre-norm counterparts.   The potential of the proposed strategy in large language models remains to be assessed.

\bigskip
{\bf Acknowlegments:} The second author would like to thank his colleagues Prof.~Tiefeng Jiang and Prof.~Jeff Yao for stimulating and useful discussions.

\bibliographystyle{plain}
\bibliography{ref}

%%%%%%%%%%%%%%%%%%%
\begin{appendix}

%%%%%%%%%%%%%%%%%%%
\section{Proofs of Results in Section 2}

%%%%%%%%%%%%%%%%%%%
\subsection{Proof of Lemma~\ref{lem:E(xi)}}

\begin{proof}
Let ``$\langle\cdot,\cdot\rangle$" denote the usual matrix inner product. For any given $Q \in \R^{n\times n}$, 
\begin{equation*}
     \|Q(X+ \alpha PXW)\|_F^2 
    =\|QX\|_F^2 + 2\alpha\langle QX, QPXW\rangle + \alpha^2 \|QPXW\|_F^2.
\end{equation*}
Under Assumption~\ref{assumpt:0} the expected value of the term linear in $W$ vanishes. Hence, 
\begin{equation*}
\E \left[ \|Q(X+\alpha PXW)\|_F^2 \right] = \|QX\|_F^2 + \alpha^2 \E[ \|QPXW\|_F^2] 
=  \|QX\|_F^2 +\alpha^2 d \sigma^2\|QPX\|_F^2,
\end{equation*}
in view of $\E[WW^T] = d\sigma^2 I$.  Hence, the expressions for $\E[\xi_1]$ and $\E[\xi_2]$ follow from substituting the matrix $Q$ by ${\Po}/{\|\Po X \|_F}$ and ${\Pop}/{\|\Pop X\|_F}$, respectively. We note that $ee^TPX=ee^TX$ if either $e^TP=e^T$ or $X=eu^T$. 

Next, to derive the lower bound for $\mu_1$.  Rewriting $\Po PX = \Po X + \Po (P-I) X$, we calculate
\begin{eqnarray*}
\mu_1^2 = \frac{\|\Po X + \Po (P-I) X\|_F^2}{\|\Po X\|_F^2} 
= 1 + \frac{2\langle\Po X,\Po(P-I)X\rangle}{\|\Po X\|_F^2} + \frac{\|\Po(P-I)X\|_F^2}{\|\Po X\|_F^2}.
\end{eqnarray*}
Since $\|\Po X\|_F = \|e^T X\|$ and $\|\Po(P-I)X\|_F = \|e^TP\Pop X\|$, we recognize that the third term is $\omega^2$.  By applying Cauchy-Schwartz inequality to the middle term, we arrive at
\[
\mu_1^2 \geq 1-2\omega+ \omega^2  = (1-\omega)^2.
\]
To obtain the upper bound of $\mu_2$ by $\delta$, we observe the equality
$$\Pop P=\Pop(P-ee^T) = \Pop P\Pop$$ which leads to
\begin{equation}\label{upper:mu2}
     \mu_2 := \frac{ \|\Pop P\Pop X\|_F}{\|\Pop X\|_F} \leq \|\Pop P\|_2 	= \delta
\end{equation}
after we invoke the inequality $\|AB\|_F \leq \|A\|_2\|B\|_F$. 
%Finally, it is easy to verify when $P$ is doubly stochastic or $X=eu^T$, then $\mu_1-1=\mu_2=0$.
\end{proof}

%%%%%%%%%%%%%%%%%%%%%%%%
\subsection{Technical Lemmas}

\begin{lemma} \label{lem:concen-eta} 
Let $\eta$ be defined as in \eqref{def:eta} and $\xi_i$, $i=1,2$, be defined as in \eqref{def:xi's} for $Y=X+\alpha PXW$.  For all $t \in [0,1]$, there holds
\begin{equation}\label{Prob-eta}
\Prob\left(|\eta| \ge t \right)    
 \leq \Prob\left(\max_{i=1,2}|\xi_i-\E[\xi_i]| \geq \gamma t \right),
\end{equation}
where
\begin{equation}\label{def:gamma}
\gamma := \frac{\E[\xi_2]^2}{\E[\xi_1] + 2\E[\xi_2]}
= \frac{(1+ \alpha^2 d\sigma^2\mu_2^2)^2}{3+\alpha^2 d\sigma^2(\mu_1^2 +2\mu_2^2)}.
\end{equation}
In particular, $\gamma \ge 1/4$ when $P$ is doubly stochastic or $X=eu^T$ for some $u\in\R^d$, and $\alpha^2 = d\sigma^2=1$.
\end{lemma}
\begin{proof}
By definition, 
\begin{equation*}
\eta = 
\frac{\E[\xi_1]}{\E[\xi_2]} - \frac{\xi_1}{\xi_2} =
\frac{\E[\xi_1] \xi_2-\xi_1 \E[\xi_2]}{\E[\xi_2] \xi_2} =
\frac{\E[\xi_1] (\xi_2 - \E[\xi_2])  - \E[\xi_2](\xi_1-\E[\xi_1])}{((\xi_2-\E[\xi_2])+\E[\xi_2]) \E[\xi_2]}.
\end{equation*}
For any $t$, $|\eta| \ge t$ implies
\begin{equation*}
t \le \frac{(\E[\xi_1] + t \E[\xi_2])|\xi_2 - \E[\xi_2]|  + \E[\xi_2]|\xi_1-\E[\xi_1]|}{\E[\xi_2]^2},
\end{equation*}
which in turn implies that for $t \in [0,1]$,
\begin{equation*}
t \le \frac{\E[\xi_1]+2\E[\xi_2]}{\E[\xi_2]^2}\max_{i=1,2}|\xi_i - \E[\xi_i]|
= \frac{1}{\gamma}\max_{i=1,2}|\xi_i - \E[\xi_i]|,
\end{equation*}
which proves \eqref{Prob-eta} with $\gamma$ defined in \eqref{def:gamma}, while the second expression in \eqref{def:gamma} follows from \eqref{E2:xi1-2}.

Finally, we know that when $P$ is doubly stochastic or $X=eu^T$, then $\mu_1=1$.  Since $\gamma$ increases monotonically with $\mu_2^2$, it attains its minimum at $\mu_2 = 0$ which gives the minimum value $1/4$. 
\end{proof}

Now we need to estimate the concentration of $\xi_1$ and $\xi_2$, both being of the form $\|A+BW\|_F^2$, that is, sum of squares of a linear transformation of a random matrix $W$ (which can also be viewed as a vector $w$).  We will invoke the following two concentration results.

%%%%%%%%%%%%%%%%%%%%%%%%%%% Vershynin Thm. 2.6.3
\begin{lemma}\label{lem:concen-aw}
(General Hoeffding’s inequality)
Let $w \in \R^q$ be a random vector whose elements are independent, mean-zero, sub-gaussian random variables with sub-gaussian norm $K$. Then for any $a\in\R^{q}$ and $\epsilon \ge 0$ there holds
\begin{equation}
\Prob \left(2 \langle a,w\rangle \geq \epsilon\|a\| \right) 
\leq \exp  \left( - c\frac{\epsilon^2}{K^2}\right),
\end{equation}
where $c$ is an absolute constant. 
\end{lemma}

%%%%%%%%%%%%%%%%%%%%%%%%%%% Vershynin Thm. 6.3.2 proof
\begin{lemma}\label{lem:concen-Bw}
(Concentration of random vectors)
Let $w \in \R^q$ be a random vector whose elements are independent, mean-zero, sub-gaussian random variables with sub-gaussian norm $K \le 1$. Then for any given $B\in\R^{p\times q}$ and any $\epsilon\in(0,1)$ there holds
\begin{equation}
\Prob \left(\|Bw\|^2 - \|B\|_F^2 \geq \epsilon\|B\|_F^2 \right) 
%\leq \exp  \left( - c\frac{\epsilon^2\|B\|_F^2}{K^2\|B\|_2^2}\right)
\leq \exp \left( - c\frac{\epsilon^2}{K^2}\right),
\end{equation}
where $c$ is an absolute constant.
\end{lemma}

Combining the above two results, we readily deduce that for some absolute constant $c$,
\begin{equation}\label{Prob:|a+Bw|}
\Prob\left(\|Bw+a\|^2 - \E\|Bw+a\|^2
\ge \epsilon(\|B\|_F^2+\|B^Ta\|)\right)
\leq 2\exp\left( - c\frac{\epsilon^2}{K^2}\right).
\end{equation}

\begin{remark}
A few remarks are in order. 
\begin{itemize}
\item
It should be noted that in these results the absolute constant $c$ is generic in the sense that it can possibly have different values in different contexts.  In doing so we avoid using multiple symbols for the sake of simplicity.
\item
Lemmas~\ref{lem:concen-aw} and \ref{lem:concen-Bw} are adopted from \cite{Vershynin2018}, see Theorem 2.6.3 and the proof of Theorem 6.3.2, respectively.  These concentration inequalities are written in a form so that the right-hand sides are dimension-free.
\item 
With regard to the values of sub-gaussian norm $K$ (also called sub-gaussian parameter) in the above lemmas,
it is known (see Example 2.5.8 in \cite{Vershynin2018}) that one can take
$K = \sigma$ for $w \sim N(0,\sigma^2I)$ and
$K = \|w\|_\infty$ if $w$ is bounded while additional absolute constants, if any, can be absorbed into $c$. 
\end{itemize}
\end{remark}

%%%%%%%%%%%%%%%%%%%%%%%%
\begin{lemma} \label{lem:concen-xi}
Let $\xi_i$, $i=1,2$, be defined as in Proposition~\ref{prop:r(X,Y)} for $Y=X+\alpha PXW$. Assume that the elements of $W$ are independent, mean-zero, sub-gaussian random variables with sub-gaussian norm $K = 1/\sqrt{d}$.  Then for $i=1,2,$ there holds
\begin{equation}\label{concen-max-xi12}
\mathbb{P}\left(|\xi_i-\E[\xi_i]| \geq \epsilon\kappa_i\right)
\leq  4 \exp\left( - c\epsilon^2d\right), 
\end{equation}
where $c$ is an absolute constant and $\kappa_i, i = 1,2$, are constants dependent on the matrices $\Pi_iX$ and $\alpha\Pi_iPX$ but independent of $W$.
\end{lemma}
\begin{proof}
This result follows directly from applying \eqref{Prob:|a+Bw|}.
\end{proof}

%%%%%%%%%%%%%%%%%%
\subsection{Proof of Theorem~\ref{thm:r-rate}}

%%%%%%%%%%%%
\begin{proof}
Lemmas~\ref{lem:concen-eta} and \ref{lem:concen-xi} together imply that $\eta$ is concentrated at zero for which the concentration probability of deviating from 0 by any fixed amount decays exponentially with $d$.  This implies that $\E[\eta]$ can be arbitrarily close to zero for $d$ sufficiently large.

Let $Y=X+\alpha PXW$ with $\sigma=1/\sqrt{d}$. Then equation \eqref{eq1:E(r)} can be written as
\begin{equation*} 
    \E[r(X,Y)] = 1+ \left(\frac{\alpha^2}{1+\alpha^2\delta^2}
    ((1-\omega)^2-\delta^2) + \Delta-\E[\eta]\right)\Sim(X),
\end{equation*}
with
\begin{equation*}
\Delta = \frac{\alpha^2}{1+\alpha^2\mu_2^2}(\mu_1^2-\mu_2^2) 
- \frac{\alpha^2}{1+\alpha^2\delta^2}((1-\omega)^2-\delta^2) > 0,
\end{equation*}
where $\Delta$ is bounded away from zero since the left-hand side of \eqref{bnd:tech} is bounded away from zero.  We note that the condition $\omega+\delta<1$ implies that the second term in $\Delta$ (without the negative sign) is positive.

Due to the increasing concentration of $\eta$ at 0 for sufficiently large $d$, there holds $\E[\eta] \le \Delta$.  By dropping the nonnegative term $\Delta-\E[\eta]$ from the above equality for $\E[r(X,Y)]$, we obtain the desired lower bound in \eqref{E5:rate} which indicates that TSE takes place in expectation.  
\end{proof}

To empirically observe the concentration of $\eta=\E[\xi_1]/\E[\xi_2]-\xi_1/\xi_2$ at zero, we conduct a set of experiments and present the results in Figure~\ref{fig:expectratio}.  Recall that $\eta$ is a function of two matrices, $X, Y \in \R^{n\times d}$.  We set $n=100$ and run $d = 10, 20, 40$.  For each $d$ value, we generate a sequence of matrices of the form $X=ev^T+tQ \in \R^{100 \times d}$ corresponding to a sequence of $t$ values in $(0,1]$, where the vector $v$ and matrix $Q$ are randomly chosen but otherwise fixed.  We note that $X$ is a perturbation to the rank-one matrix $ev^T$ with $t$ scaling  the size of the perturbation.  For each $X$, we generate $Y = X + PXW$ with 50 random samples of $W$ drawn from $\mathcal{N}(0, I/d)$, and then compute the average value of $\eta$ over the 50 samples. The attention matrix $P$ is computed using the softmax formula \eqref{softmax}.

%%%%%%%%%%%%%%%%%%%%%%%%%%%%%%%%%%%%%
%
\begin{figure}[htp]
\centering
\includegraphics[width=.9\textwidth,trim=0 0 0 0, clip]{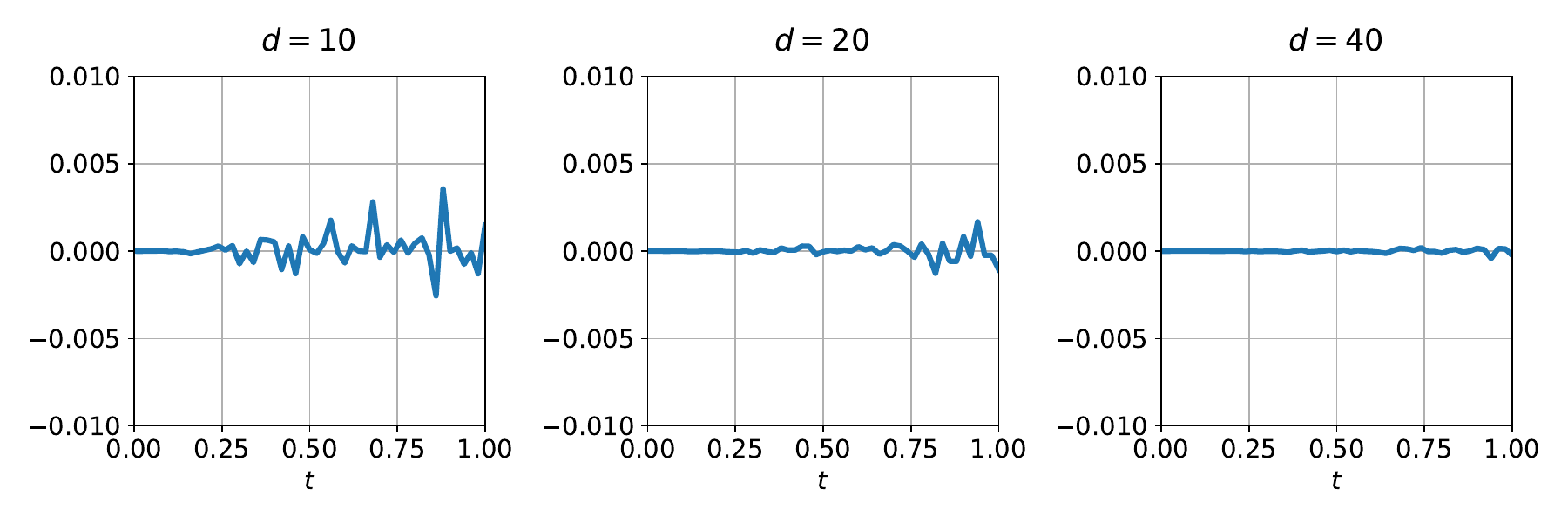}
\caption{
Average values of $\eta = \E [\xi_1] / \E [\xi_2] -  \xi_1/\xi_2$ over 50 samples of $W \in \R^{d\times d}$ for $Y = X + PXW$ where  $X$ = $ev^T+tQ$ and P is from the softmax formula.
}
\label{fig:expectratio}
\end{figure}

As we see from Figure~\ref{fig:expectratio}, (a) when $t$ is small (i.e., $X$ close to rank-one), (sample mean) $\E[\eta]$ is close to 0 even for $d=10$; (b) on the other hand, when $t$ is close to 1, $\E[\eta]$ varies more significantly; and (c) as $d$ increases, $\E[\eta]$ becomes closer to 0 even for larger $t$ values.  From these experiments, we observe that as $d$ becomes larger, indeed $\eta$ concentrates more at zero so that $\E[\eta]$ approaches 0.

%%%%%%%%%%%%%%%%%%%%%%%%%%%%%%%%%%%
\subsection{Proof of Proposition~\ref{prop:multihead}}

\begin{lemma} \label{lem:E(xi) in Apdx in multihead}
Given $X \in \R^{n\times d}$, $P_k \in \R^{n\times n}$ and $W_k \in \R^{d \times d/h}$, for $k=1,\cdots,h$, and $\alpha>0$, let
\[
Y = X+ \alpha [P_1XW_1\;\; P_2XW_2\;\; ...\;\; P_hXW_h].
\]
Then under Assumption~\ref{assumpt:0}(1), 
\begin{equation} \label{E6:xi1-2 in Apdx}
   \E[\xi_i] \equiv \E\left[\frac{\|\Pi_i Y\|_F^2}{\|\Pi_i X\|_F^2}\right] 
   = 1 + {\alpha^2 d \sigma^2} \bar{\mu}_i^2,  \;\; i=1,2,
\end{equation}
where  
\begin{equation*}
\bar{\mu}_i^2 := \frac{1}{h} \sum_{k=1}^h\frac{\|\Pi_i P_k X\|_F^2}{\|\Pi_i X\|_F^2}, \;\; i=1,2.
\end{equation*}
\end{lemma}
\begin{proof}
The proof follows a similar line as in that of Lemma~\ref{lem:E(xi)}.
First, we note that under Assumption~\ref{assumpt:0}(1), 
$\E[W_k]=0$ and $\E[W_kW_k^T] = (d/h)\sigma^2I_d$ for $k=1,\cdots,h$.

Given $Q \in \R^{n\times n}$, noting that expected values of terms linear in $W_k$ all vanish, we have
\[
\E \left[\|Q(X+\alpha[P_1XW_1, \cdots, P_hXW_h])\|_F^2 \right] = 
\|QX\|_F^2 + \alpha^2\sum_{k=1}^h\E \left[\|QP_kXW_k\|_F^2 \right],
\] 
In view of $\E[W_kW_k^T] = (d/h)\sigma^2I$, we obtain
\begin{eqnarray*}
\E \left[ \|Q(X+\alpha[P_1XW_1, \cdots, P_hXW_h])\|_F^2 \right] 
=  \|QX\|_F^2 +\alpha^2 d\sigma^2\;\frac{1}{h}\sum_{k=1}^h\|QP_kX\|_F^2.
\end{eqnarray*}
Finally, the expressions for $\E[\xi_i]$ follow from substituting the matrix $Q$ by $\Pi_i/\|\Pi_i X \|_F$, for  $i=1,2$, which completes the proof.
\end{proof}

Now Proposition~\ref{prop:multihead} follows directly from Lemma~\ref{thm:E-factor} and Lemma~\ref{lem:E(xi) in Apdx in multihead}.

\end{appendix}
\end{document}